\documentclass[11pt]{article}
\usepackage[lmargin=1in,rmargin=1in, bmargin=1.25in, tmargin=1.25in]{geometry}

\usepackage[numbers]{natbib}

\usepackage{amsmath,amssymb,amsfonts}
\usepackage{enumerate}
\usepackage{amsthm,amssymb,amsfonts}
\usepackage{color}
\usepackage{xcolor}
\usepackage{pgf}
\usepackage{caption}
\usepackage{subcaption}

\newcommand{\jmlr}{{\cite{gourdeau2021hardness}}}

\newcommand{\R}{\mathbb{R}}
\newcommand{\N}{\mathbb{N}}
\newcommand{\abs}[1]{ \left\vert #1\right\vert }
\newcommand{\set}[1]{\left\{#1\right\}}
\newcommand{\eval}[2]{\underset{{#1}}{\mathbb{E}}\left[#2\right]}
\newcommand{\Prob}[2]{\underset{{#1}}{\Pr}\left(#2\right)}
\newcommand{\st}{\;.\;}
\newcommand{\roblossc}{\mathsf{R}^C_\rho}
\newcommand{\roblosse}{\mathsf{R}^E_\rho}
\newcommand{\robloss}{\mathsf{R}_\rho}
\newcommand{\roblog}{\mathsf{R}_{\log(n)}}

\newcommand{\satlog}{\mathsf{S}_{\log(n)}}
\newcommand{\satrho}{\mathsf{S}_{\rho}}
\newcommand{\satnot}{\mathsf{S}_{0}}
\newcommand{\boolhc}{\set{0,1}^n}

\newcommand{\poly}{\text{poly}}

\newcommand{\given}{\;|\;}

\newcommand{\A}{\mathcal{A}}
\newcommand{\C}{\mathcal{C}}

\newcommand{\D}{\mathcal{D}}

\newcommand{\E}{\mathcal{E}}

\renewcommand{\H}{\mathcal{H}}
\newcommand{\X}{\mathcal{X}}

\newcommand{\matching}{\mathcal{M}}
\newcommand{\asst}{\mathcal{A}_{\matching}}
\newcommand{\varM}{I_{\matching}}
\newcommand{\MonConj}{\textsf{MON-CONJ}}

\newcommand{\dl}{$\mathsf{DL}$}

\newtheorem{theorem}{Theorem}

\newtheorem{lemma}[theorem]{Lemma}
\newtheorem{corollary}[theorem]{Corollary}

\newtheorem{definition}[theorem]{Definition}

\newtheorem{remark}[theorem]{Remark}

\title{Sample Complexity Bounds for Robustly Learning Decision Lists against Evasion Attacks}
\author{
Pascale Gourdeau,
Varun Kanade, 
Marta Kwiatkowska, and
James Worrell\\
University of Oxford
}

\begin{document}
\maketitle
\begin{abstract}
A fundamental problem in adversarial machine learning is to quantify
how much training data is needed in the presence of evasion attacks. In this paper
we address this issue within the framework of PAC learning, focusing on  
the class of decision lists.  Given that distributional assumptions are essential in the adversarial setting, we work with probability distributions on the input data that satisfy a
Lipschitz condition: nearby points have similar probability.  Our key results
illustrate that the adversary's budget (that is, the number of bits it can perturb
on each input) is a fundamental quantity in determining the sample complexity of
robust learning.
Our first main result is a sample-complexity lower bound: the class of monotone conjunctions (essentially the simplest non-trivial hypothesis class on the Boolean hypercube) and any superclass  has sample complexity at least exponential in the adversary's budget. Our second main result is a corresponding upper bound: for every fixed $k$ the class of $k$-decision lists has polynomial sample complexity against a $\log(n)$-bounded adversary. This sheds
further light on the question
of whether
an efficient PAC learning algorithm can always be used as an efficient
$\log(n)$-robust learning algorithm under the uniform distribution.

\end{abstract}

\section{Introduction}

Adversarial machine learning has been extensively studied in recent years, first with spam filtering in
\cite{dalvi2004adversarial,lowd2005adversarial,lowd2005good}, and then
when the notion of \emph{adversarial examples} was introduced by
\citet{szegedy2013intriguing}, and independently noticed by
\citet{biggio2013evasion}.  Various settings to study adversarial
machine learning guarantees (and impossibility results) have emerged
in the literature since.  The most common distinction, presented by
\citet{biggio2017wild}, differentiates between attacks at training
time, known as \emph{poisoning attacks}, and attacks at test time,
called \emph{evasion attacks}.

In the context of evasion attacks, a misclassification by a model has
been defined in various ways, and sometimes regrettably referred to by
the same terminology.
\citet{dreossi2019formalization,diochnos2018adversarial,gourdeau2021hardness}
offer thorough discussions on the subject.  We will focus on the
\emph{exact-in-the-ball} notion of robustness (also known as
\emph{error region} risk in \cite{diochnos2018adversarial}), which
necessitates a ground truth function.  Briefly, the exact-in-the-ball
notion of robustness requires a hypothesis to be correct with respect
to the ground truth in a perturbation region around each test point.
Note that, in this case, the ground truth must be specified on all
input points in the perturbation region.  By contrast, the
constant-in-the-ball notion of robustness (which is also known as
\emph{corrupted input} robustness) is concerned with the stability of
the hypothesis to perturbations in the input, and requires that the
label produced by the hypothesis remain constant in the perturbation
region, meaning that we only need access to the test point labels.

The hardness of robust classification has been explored from both a computational complexity and a statistical viewpoint, see for e.g., \cite{bubeck2019adversarial,montasser2019vc}.
In this paper, we focus on the Boolean hypercube $\boolhc$ as our input space and study the information-theoretic complexity of robust learning by exhibiting sample complexity upper and lower bounds that depend on an \emph{adversarial budget}, i.e., the number of bits an adversary is allowed to flip at test time, thus illustrating that the adversarial
budget is a fundamental quantity in determining the sample
complexity of robustly learning important concept classes.

\subsection{Our Contributions}

Our work builds on the work of \citet{gourdeau2019hardness} and its extended version \cite{gourdeau2021hardness}.
Our results hold for the \emph{exact-in-the-ball}  robustness to evasion attacks, and are outlined below.

\textbf{Robust Learning of Decision Lists:}
As shown in \cite{gourdeau2021hardness}, \emph{efficient, exact-in-the-ball} robust learning is not possible without distributional assumptions on the training data.\footnote{This is in contrast to PAC learning, which is distribution-free.} 
We follow their line of work and establish the sample-efficient robust learnability of decision lists against a $\log(n)$-bounded adversary under log-Lipschitz distributions, which include the uniform and product distributions. The algorithms we use to show such upper bounds are called $\rho$-robust learning algorithms, where $\rho$ is the allowed perturbation budget for an adversary.
In proving our first result we obtain
an isoperimetric bound that may be of independent interest: for a CNF
formula $\varphi$ we give an upper bound on the number of points in
the Boolean hypercube within a given Hamming distance to a
satisfying assignment of $\varphi$.
An analogue result was shown only for \emph{monotone} decision lists in \jmlr.
More importantly, \citet{gourdeau2021hardness} suggested the following open problem:
\begin{center}
\emph{Let $\A$ be a sample-efficient (potentially proper) PAC-learning algorithm for concept class $\C$. Is $\A$ also a sample-efficient $\log(n)$-robust learning algorithm for $\C$ under the uniform distribution?}
\end{center}
So far, all the concept classes that have been studied point towards a positive answer to this question. 
As log-Lipschitz distributions subsume the uniform distribution, our result thus adds to the body of positive evidence for this problem.

\textbf{An Adversarial Sample Complexity Lower Bound:} 
To complement the above result, we show that any $\rho$-robust learning algorithm for monotone conjunctions must have a sample complexity that is exponential in the number $\rho$ of bits an adversary is allowed to flip during an evasion attack.
Previously, \citet{gourdeau2021hardness}  showed that there does not exist such an algorithm with polynomial sample complexity against an adversary that can perturb $\omega(\log(n))$ bits of the input.

\subsection{Related Work}

The inevitability of adversarial examples under the constant-in-the-ball 
definition of robustness has been extensively studied, see for e.g., \cite{fawzi2016robustness,fawzi2018adversarial,fawzi2018analysis,gilmer2018adversarial,shafahi2018adversarial,tsipras2019robustness,ilyas2019adversarial}.
We first outline related work on sample complexity lower bounds
for robust learning. \citet{bhagoji2019lower} work with the 
constant-in-the-ball definition of robustness and use an optimal transport 
cost function to derive lower bounds for learning classes with 
labels that come from a mixture of Gaussian distributions. \citet{montasser2019vc} also
use this notion of robustness to show a lower bound that depends on a complexity
measure adapted to robustness from the shattering dimension of a concept class. 
Closer to our  work, \citet{diochnos2019lower,diochnos2020lower} exhibit 
lower bounds for the exact-in-the-ball robust risk. They focus on a family of 
concentrated  distributions, Normal Lévy families, which include, for e.g., the 
Gaussian distribution on $\R^n$ and product distribution of dimension $n$ 
under the Hamming distance.\footnote{We work with the uniform distribution, 
which is a special case of product distributions.} Instead of
looking at a specific class of functions, they allow any concept class that 
contain concepts that have small enough ($2^{-\Theta(n)}$) standard error 
with respect to each other, and so would be indistinguishable for sufficiently 
small samples. Note that monotone conjunctions satisfy this property. When
considering the Boolean hypercube and an adversary that can perturb $\rho$ 
bits of the input, they get that any robust PAC learning algorithm for their 
robust learning setting requires a sample of size $2^{\Omega(\rho^2/n)}$. 
Note that this lower bound is non trivial only when considering adversaries
that can perturb $\sqrt{n}$ bits or more, while we show a lower bound that 
is strictly \emph{exponential} in the adversary's budget (though for  slightly
more restricted concept classes), and thus meaningful for a wider class of 
adversaries.

In terms of sample complexity upper  bounds,  \citet{montasser2019vc} show
sample complexity upper bounds that are linear (ignoring log factors) in the 
VC dimension and the dual VC 
dimension of a concept class under the constant-in-the-ball notion of 
robustness, yielding an exponential upper bound in the VC dimension. 
As noted in 
{\jmlr}, their techniques do not apply to the exact-in-the-ball setting, which is
studied for evasion attacks notably in \cite{diochnos2018adversarial,mahloujifar2019can,mahloujifar2019curse,gourdeau2019hardness,gourdeau2021hardness}.
The work of \citet{diochnos2018adversarial} addresses the ability of an adversary 
to cause a blow up the adversarial error with respect to the standard error.
For instance, they show that, under the uniform distribution, a 
$O(\sqrt{n})$-bounded adversary can cause
the probability of a misclassification to be $1/2$ given that the standard 
error is $0.01$ for any learning problem. These results are extended by
 \citet{mahloujifar2019curse} for a wider family of distributions. Finally, 
 \citet{gourdeau2021hardness} exhibit sample complexity upper bounds
 for the robust learnability of a variety of concept classes (parities, 
 \emph{monotone} decision lists, and 
 decision trees) under log-Lipschitz distributions for various adversarial budgets.

\section{Problem Set Up}

In this section, we will first recall two definitions of robustness that have been widely used in the literature, and formalize the notion of robustness thresholds in the robust PAC-learning framework. 
We will then review relevant concept classes for this paper, as well as log-Lipschitz distributions, which were  introduced in \cite{awasthi2013learning} and will be the focus of our results.

\subsection{Robust Learning }

We work in the PAC learning framework of \citet{valiant1984theory} (see Appendix~\ref{app:pac}), but where the (standard) risk function is replaced by a \emph{robust} risk function. 
Since we focus on the Boolean hypercube $\boolhc$ as the input space, the only relevant notion of distance between points is the Hamming distance (denoted $d_H$), i.e., the number of bits that differ between two points.
Thus, the adversary's perturbation budget will be the number of bits of the input the adversary is allowed to flip to cause a misclassification.
We will use the \emph{exact-in-the-ball} definition of robust risk (which is called \emph{error-region} risk in \cite{diochnos2018adversarial}). 
Given respective hypothesis and target functions
$h,c:\mathcal{X}\rightarrow\{0,1\}$, distribution $D$ on
$\mathcal{X}$, and robustness parameter $\rho\geq 0$, the
exact-in-the-ball robust risk of $h$ with respect to $c$ is defined as
$\roblosse(h,c)=\Prob{x\sim D}{\exists z\in B_\rho(x):h(z)\neq c(z)}$, where $B_\rho(x)=\{z\in\boolhc\given d_H(x,z)\leq\rho\}$.
This is in contrast to the more widely-used  \emph{constant-in-the-ball} risk function (also called \emph{corrupted-instance} risk from the work of \citet{feige2015learning}) $\roblossc(h,c)=\Prob{x\sim D}{\exists z\in B_\rho(x):h(z)\neq c(x)}$ where the hypothesis is required to be constant in the perturbation region in addition to being correct with respect to the unperturbed point's label $ c(x)$.

Both \citet{diochnos2018adversarial} and \citet{gourdeau2021hardness} offer a thorough discussion on the advantages and drawbacks of the two notions of robust risk. 
We will study the \emph{exact-in-the-ball} robust risk, as our learning problems have considerable probability mass near the decision boundary. 
Thus  it makes sense to consider the faithfulness of the hypothesis with respect to the target function.
The exact-in-the-ball robust risk also has various advantages: if the distribution is supported on the whole input space (e.g., the uniform distribution), exact learnability implies robust learnability and the target concept is always the robust risk minimizer.\footnote{This is not necessarily the case with the constant-in-the-ball definition \cite{gourdeau2021hardness}.}
We have from \citet{gourdeau2021hardness} the following definition of robust learnability with respect to the exact-in-the-ball robust risk. 
Note that we will henceforth drop the superscript and simply use $\robloss$ to denote the exact-in-the-ball robust risk.

\begin{definition}%[\citet{gourdeau2021hardness}]
\label{def:robust-learning}
Fix a function $\rho:\N\rightarrow\N$. We say that an algorithm $\A$
\emph{efficiently} $\rho$-\emph{robustly learns} a concept class $\C$
with respect to distribution class $\mathcal{D}$ if there exists a
polynomial $\poly(\cdot,\cdot,\cdot,\cdot)$ such that for all
$n\in\mathbb{N}$, all target concepts $c\in \C_n$, all distributions
$D \in \mathcal{D}_n$, and all accuracy and confidence parameters
$\epsilon,\delta>0$, if $m \geq
\poly(n,1/\epsilon,1/\delta,\text{size}(c))$, whenever $\A$ is given access to
a sample $S\sim D^m$ labelled according to $c$, it outputs a polynomially evaluable function
$h:\{0,1\}^n\rightarrow\{0,1\}$ such
that $\Prob{S\sim D^m}{\robloss(h,c)<\epsilon}>1-\delta$.
\end{definition}

\subsection{Concept Classes and Distribution Families}

Our work uses formulas in the conjunctive normal form (CNF) to show the robust learnability of decision lists. 
This concept class was proposed and shown to be PAC learnable by \citet{rivest1987learning}.
Formally, given the maximum size $k$ of a conjunctive clause, a decision list $f\in k$-{\dl} is a list $(K_1,v_1),\dots,(K_r,v_r)$ of pairs
where $K_j$ is a term in the set of all conjunctions of size at most $k$ with literals drawn from $\set{x_1,\bar{x_1},\dots,x_n,\bar{x_n}}$, $v_j$ is a value in $\set{0,1}$, and $K_r$ is $\mathtt{true}$.
The output $f(x)$ of $f$ on $x\in\boolhc$ is $v_j$, where $j$ is the least index such that the conjunction $K_j$ evaluates to $\mathtt{true}$.

Given $k,n\in\N$, we denote by $\varphi$ a $k$-CNF on $n$ variables, where $k$ refers to the size of the largest clause in $\varphi$.
Note that the class $\MonConj$ of monotone conjunctions, where each variable appears as a positive literal, is a subclass of 1-CNF formulas.
Moreover, since decision lists generalize formulas in disjunctive normal form (DNF) and conjunctive normal form, in the sense that $k$-CNF $\cup$ $ k$-DNF $\subseteq k$-DL,
 a robust learnability result for $k$-{\dl} holds for $k$-CNF and $k$-DNF as well.
We refer the reader to Appendix~\ref{app:formulas} for more background on conjunctions and $k$-CNF formulas.

For a formula $\varphi$, we will denote by $\satnot(\varphi)$ the probability $\Prob{x\sim D}{x\models \varphi}$ that $x$ drawn from distribution $D$ results in a satisfying assignment of $\varphi$.
We will also denote the probability mass  $\Prob{x\sim D}{\exists z\in B_{\rho}(x) \st z \models \varphi}$ of the $\rho$-expansion of a satisfying assignment  by $\satrho(\varphi)$.

Our robust learnability results will hold for a class of sufficiently smooth distributions, called log-Lipschitz distributions, originally introduced by \citet{awasthi2013learning}: 

\begin{definition}
A distribution  $D$ on $\boolhc$ is said to be $\alpha$-$\log$-Lipschitz if 
for all input points $x,x'\in \boolhc$, if $d_H(x,x')=1$, then $|\log(D(x))-\log(D(x'))|\leq\log(\alpha)$.
\end{definition}

Neighbouring points in $\boolhc$ have probability masses that differ
by at most a multiplicative factor of $\alpha$ under
$\alpha$-$\log$-Lipschitz distributions.  
The decay of probability mass along a chain of neighbouring points is thus at most exponential; not having sharp changes to the underlying distribution is a very natural assumption, and one weaker than many often make in the literature.
Note that features are allowed a small dependence between each other and, by construction,
log-Lipschitz distributions are supported on the whole input space.
Notable examples of log-Lipschitz distributions are the uniform
distribution (with parameter $\alpha=1$) and the class of product
distributions with bounded means.

\section{The $\log(n)$-Expansion of Satisfying Assignments for $k$-CNF Formulas }
\label{sec:log-exp}

In this section, we show that, under log-Lipschitz distributions, the
probability mass of the $\log(n)$-expansion of the set of satisfying
assignments of a $k$-CNF formula can be bounded above by an arbitrary
constant $\varepsilon>0$, given an upper bound on the probability of a
satisfying assignment.  The latter bound is polynomial in
$\varepsilon$ and $1/n$.  While this result is of general interest,
our goal is to prove the efficient robust learnability of decision
lists against a $\log(n)$-bounded adversary.  Here the relevant fact
is that, given two decision lists $c,h\in k$-DL, the set of inputs in
which $c$ and $h$ differ can be written as a disjunction of
quadratically many (in the combined length of $c$ and $h$) $k$-CNF
formulas.  The $\log(n)$-expansion of this set is then the set of
inputs where a $\log(n)$-bounded adversary can force an error at test
time.  This is the main technical contribution of this paper, and the
theorem is stated below.  The combinatorial approach, below, vastly
differs from the approach of \cite{gourdeau2021hardness} in the
special case of monotone $k$-DL, which relied on facts about
propositional logic.

\begin{theorem}
\label{thm:k-cnf}
Suppose that $\varphi\in k$-CNF and let $D$ be an
$\alpha$-log-Lipschitz distribution on the valuations of $\varphi$.
Then there exist constants $C_1,C_2,C_3,C_4\geq 0$ that depend on
$\alpha$ and $k$ such that if the probability of a satisfying
assignment satisfies
$\satnot(\varphi) <
C_1\varepsilon^{C_2}\min\set{\varepsilon^{C_3},n^{-C_4} } $, then the
$\log(n)$-expansion of the set of satisfying assignments has
probability mass bounded above by $\varepsilon$.
\end{theorem}

\begin{corollary}
\label{cor:k-dl}
The class of $k$-decision lists is efficiently $\log(n)$-robustly learnable under log-Lipschitz distributions.
\end{corollary}

The proof of Corollary~\ref{cor:k-dl} is similar to Theorem~24 in \cite{gourdeau2021hardness}, and is included in Appendix~\ref{app:log-exp}.
We note that it is imperative that the constants $C_i$ do not depend on the learning parameters or the input dimension, as the quantity $C_1\varepsilon^{C_2}\min\set{\varepsilon^{C_3},n^{-C_4} }$ is directly used as the accuracy parameter in the (proper) PAC learning algorithm for decision lists, which is used as a black box.

To prove Theorem~\ref{thm:k-cnf}, we will need several lemmas outlined
below, which are either taken directly or slightly adapted from
\cite{gourdeau2021hardness}.
The first is an adaptation of Lemma~17 in \cite{gourdeau2021hardness} for conjunctions, which was originally stated for decision lists:

\begin{lemma}
\label{lemma:conj-err-length}
Let $\varphi$ be a conjunction and let $D$ be an $\alpha$-log-Lipschitz distribution. 
If $\Prob{x\sim D}{x\models \varphi}<\left(1+\alpha\right)^{-d}$, then $\varphi$ is a conjunction on at least $d$ variables.
\end{lemma}

The second result, which states an upper bound on the expansion of satisfying assignments for conjunctions, will be used for the base case of the induction proof. 

\begin{lemma}
\label{lemma:rob-risk-dl}
Let $D$ be an $\alpha$-$\log$-Lipschitz distribution on the
$n$-dimensional Boolean hypercube and let $\varphi$ be a 
conjunction of $d$ literals.
Set $\eta=\frac{1}{1+\alpha}$.
Then for all $0<\varepsilon<1/2$,
if $d\geq \max\left\{
  \frac{4}{\eta^2}\log\left(\frac{1}{\varepsilon}\right) ,
  \frac{2\rho}{\eta} \right\}$, then 
$\Prob{x\sim D}{\left(\exists y \in B_\rho(x) \cdot y \models
    \varphi\right)} \leq \varepsilon$.
\end{lemma}

Finally, we will use the following lemma, which will be used in the inductive step of the induction proof.

\begin{lemma}
  \label{lemma:rob-risk-dl-2}
  Let $\varphi$ be a $k$-CNF formula that has a set of variable-disjoint clauses of size $M$.
  Let $D$ be an $\alpha$-log-Lipschitz distribution on
  valuations for $\varphi$.  Let $0<\varepsilon<1/2$ be arbitrary
  and set $\eta:=\left({1+\alpha}\right)^{-k}$.  If
 $M \geq \max \left\{
  \frac{4}{\eta^2}\log\left(\frac{1}{\varepsilon}\right),
  \frac{2\rho}{\eta}\right \}$ then
  $\Prob{x\sim D}{\exists y \in B_\rho(x) \cdot y \models \varphi}
    \leq \varepsilon$.

\end{lemma}

We are now ready to prove Theorem~\ref{thm:k-cnf}.  The main idea
behind the proof is to consider a given $k$-CNF formula $\varphi$ and
distinguish two cases: (i) either $\varphi$ contains a
sufficiently-large set of variable-disjoint clauses, in which case the
adversary is not powerful enough to make $\varphi$ satisfied by
Lemma~\ref{lemma:rob-risk-dl-2}; or (ii) we can rewrite $\varphi$ as the
disjunction of a sufficiently small number of $(k-1)$-CNF formulas,
which allows us to use the induction hypothesis to get the desired
result.  The final step of the proof is to derive the constants
mentioned in the statement of Theorem~\ref{thm:k-cnf}.
 
\begin{proof}[Proof of Theorem~\ref{thm:k-cnf}]

We will use the lemmas above and restrictions on $\varphi$ to show the following.\\

\emph{Induction hypothesis:}
Suppose that $\varphi\in (k-1)$-CNF and let $D$ be an $\alpha$-log-Lipschitz distribution on the valuations of $\varphi$. 
Then there exists constants $C_1,C_2,C_3,C_4\geq 0$ that depend on $\alpha$ and $k$ and satisfy $C_3\geq\frac{\eta}{2}C_4$ such that if $\satnot(\varphi) < C_1\varepsilon^{C_2}\min\set{\varepsilon^{C_3},n^{-C_4} }$, then $\satlog(\varphi)\leq \varepsilon $.\\

\emph{Base case:} This follows from Lemmas~\ref{lemma:conj-err-length} and \ref{lemma:rob-risk-dl}.
Set $\eta$ to $(1+\alpha)^{-1}$, and $C_1=1$, $C_2=0$, $C_3=\frac{4}{\eta^2}$ and $C_4=\frac{2}{\eta}$. 
Note that $C_3\geq\frac{\eta}{2}C_4$.\\

\emph{Inductive step:} 
Suppose $\varphi\in k$-CNF and let $D$ be an $\alpha$-log-Lipschitz distribution on the valuations of $\varphi$. 
Set $\eta=(1+\alpha)^{-k}$.
Let $C_1',C_2',C_3',C_4'$ be the constants in the induction hypothesis for $\varphi'\in (k-1)$-CNF.
Set the following constants:
\begin{align*}
&C_1= C_1'2^{-k(C_2'+C_3')}\\
& C_2=C_2'+C_3' \\
&C_3= \frac{8}{\eta^2}\max\set{C_2',C_3'}\\
&C_4= \frac{2}{\eta}\max\set{C_2',C_3'}\enspace,
\end{align*}
and note that these are all constants that depend on $k$ and $\alpha$ by the induction hypothesis, and that $C_3\geq\frac{\eta}{2}C_4$.

Let $\satnot (\varphi)< C_1 \varepsilon^{C_2}\min\set{\varepsilon^{C_3},n^{-C_4} }$.
Let $\mathcal{M}$ be a maximal set of clauses of $\varphi$ such that no two clauses contain the same variable.
Denote  by $I_\matching$ the indices of the variables in $\matching$ and let $M=\max\set{\frac{4}{\eta^2}\log\frac{1}{\varepsilon},\frac{2}{\eta}\log n}$. \\

We distinguish two cases:\\

(i) $\abs{\matching}\geq M$:  

We can then invoke Lemma~\ref{lemma:rob-risk-dl-2} and guarantee that $\satlog\leq \varepsilon$, and we get the required result.\\

(ii) $\abs{\matching}<M$:

Then let $\asst$ be the set of assignments of variables in $\matching$, i.e. $a\in\asst$ is a function $a:I_{\matching}\rightarrow \set{0,1}$, which represents a partial assignment of variables in $\varphi$.
We can thus rewrite $\varphi$ as follows:
\begin{equation*}
\varphi \equiv \bigvee_{a\in\asst} \left( \varphi_a \wedge \bigwedge_{i\in\varM}  l_i \right)\enspace,
\end{equation*}
where $\varphi_a$ is the restriction of $\varphi$ under assignment $a$
and $l_i$ is $x_i$ in case $a(i)=1$ and $\bar{x_i}$ otherwise.  For
short, denote by $\varphi_a'$ the formula
$\varphi_a \wedge \bigwedge_{i\in\varM} l_i$.  By the
  maximality of $\mathcal{M}$ every clause in $\varphi$ mentions some
  variable in $\mathcal{M}$, and hence $\varphi_a'$ is $(k-1)$-CNF.
Moreover, the formulas $\varphi_a'$ are disjoint, in the sense that if
some assignment $x$ satisfies $\varphi_a'$, it will not satisfy
another $\varphi_b'$ for a distinct index $b$.
Note also that 
$$A_{n,\varepsilon}:=\abs{\asst}\leq
2^k\max\set{\left(\frac{1}{\varepsilon}\right)^{4/\eta^2},
  n^{2/\eta}}\enspace
\, . $$
Thus, 
\begin{equation}
\label{eqn:std-risk-split}
\satnot(\varphi) %= \Prob{x\sim D}{x\models \varphi}
=\sum_{a\in\asst}\Prob{x\sim D}{x\models \varphi_a' }
= \sum_{a\in\asst} \satnot(\varphi_a')
\enspace.
\end{equation}
By the induction hypothesis, we can guarantee that if
\begin{align}
\label{eqn:std-risk}
 \satnot(\varphi_a') < \;&C_1'  \left(\frac{\varepsilon}{A_{n,\varepsilon}}\right)^{C_2'}
 \min \set{\left(\frac{\varepsilon}{A_{n,\varepsilon}}\right)^{C_3'},n^{-C_4'} } 
\end{align} 
 for all $\varphi_a'$ then the $\log(n)$-expansion $\satlog(\varphi)$ can be bounded as follows:
\begin{align*}
\satlog(\varphi)&=\Prob{x\sim D}{\exists z\in B_{\log n}(x) \st z \models \varphi}\\
&=\sum_{a\in\asst}  \Prob{x\sim D}{\exists z\in B_{\log n}(x) \st z \models \varphi_a'} \\
&\leq \sum_{a\in\asst}   \frac{\varepsilon}{A_{n,\varepsilon}} \tag{I.H.} \\
&=\varepsilon \enspace.
\end{align*}

By Equation~\ref{eqn:std-risk-split}, the upper bound  $\satnot (\varphi)< C_1 \varepsilon^{C_2}\min\set{\varepsilon^{C_3},n^{-C_4} }$ on the probability of a satisfying assignment for $\varphi$ implies an upper bound $\satnot (\varphi_a')< C_1 \varepsilon^{C_2}\min\set{\varepsilon^{C_3},n^{-C_4} }$ on the probability of the restrictions $\varphi_a'$.
Thus it only remains to show that the condition on $\satnot (\varphi)$  implies that Equation~\ref{eqn:std-risk} holds.

Let us rewrite the RHS of Equation~\ref{eqn:std-risk} as follows, where each of the equations is a stricter condition on $\satnot (\varphi_a')$ than its predecessor:
\begin{align*}
& C_1'  \left(\frac{\varepsilon}{A_{n,\varepsilon}}\right)^{C_2'}\min \set{\left(\frac{\varepsilon}{A_{n,\varepsilon}}\right)^{C_3'},n^{-C_4'} } \\
& \geq C_1'  \left(\frac{\varepsilon}{2^k}\right)^{C_2'}\min\set{\varepsilon^{4C_2'/\eta^2},n^{-2C_2'/\eta}} \min \set{\left(\frac{\varepsilon^{1+4/\eta^2}}{2^k}\right)^{C_3'},\left(\frac{\varepsilon n^{-2/\eta}}{2^k}\right)^{C_3'},n^{-C_4'} } \\
& = C_1'  \left(\frac{\varepsilon}{2^k}\right)^{C_2'}\min\set{\varepsilon^{4C_2'/\eta^2},n^{-2C_2'/\eta}}  \min \set{\left(\frac{\varepsilon^{1+4/\eta^2}}{2^k}\right)^{C_3'},\left(\frac{\varepsilon n^{-2/\eta}}{2^k}\right)^{C_3'} }\\
& = C_1'  2^{-k(C_2'+C_3')}\varepsilon^{C_2'+C_3'}\min\set{\varepsilon^{4C_2'/\eta^2},n^{-2C_2'/\eta}} \min \set{\varepsilon^{4C_3'/\eta^2}, n^{-2C_3'/\eta} }\\
&\geq C_1'  2^{-k(C_2'+C_3')}\varepsilon^{C_2'+C_3'} \min\set{\varepsilon^{8C_2'/\eta^2},n^{-4C_2'/\eta},\varepsilon^{8C_3'/\eta^2}, n^{-4C_3'/\eta} }\\
&= C_1'  2^{-k(C_2'+C_3')}\varepsilon^{C_2'+C_3'} \min\set{\varepsilon^{8\max\set{C_2',C_3'}/\eta^2},n^{-4\max\set{C_2',C_3'}/\eta} }\\
&= C_1 \varepsilon^{C_2}\min\set{\varepsilon^{C_3},n^{-C_4}}
\enspace,
\end{align*}
where the first step is by definition of $A_{n,\varepsilon}$, the second from the induction hypothesis, which guarantees $C_3'\geq\frac{\eta}{2}C_4'$, and the fourth from the property $\min\set{a,b}\cdot\min\set{c,d}\geq\min\set{a^2,b^2,c^2,d^2}$.
Finally, the last equality follows by the definition of the $C_i$'s.

Note that we set $\eta=(1+\alpha)^{-k}$ to be able to apply Lemma~\ref{lemma:rob-risk-dl-2} in the first part of the inductive step.
Then, $A_{n,\epsilon}$ is a function of $\eta=(1+\alpha)^{-k}$.
When we consider the distribution on the valuations of the restriction $\varphi_a'$, we still operate with an $\alpha$-log-Lipschitz distribution on its valuations, by log-Lipschitz facts (see Appendix~\ref{app:log-lipschitz}).

\emph{Constants.}
We want to get explicit constants $C_1,C_2,C_3$ and $C_4$ as a function of $k$ and $\eta
$.
Note that $\eta=(1+\alpha)^{-k}$ is dependent on $k$.  
Let us recall the recurrence system from the inductive step:
\begin{align*}
&C_1^{(k)}= C_1^{(k-1)}2^{-k(C_2^{(k-1)}+C_3^{(k-1)})} \\
& C_2^{(k)}=C_2^{(k-1)}+C_3^{(k-1)} \\
&C_3^{(k)}= \frac{8}{\eta^2}\max\set{C_2^{(k-1)},C_3^{(k-1)}} \\
&C_4^{(k)}= \frac{2}{\eta}\max\set{C_2^{(k-1)},C_3^{(k-1)}}\enspace.
\end{align*}
It is easy to see that $C_3^{(k)}\geq C_2^{(k)}$ for all $k\in\N$.
If we fix $\eta=(1+\alpha)^{-k}$ at each level of the recurrence,  we can now consider the following recurrence system, which dominates the previous one:
\begin{align*}
C_1^{(k)}&= C_1^{(k-1)}2^{-2kC_3^{(k-1)}}&C_2^{(k)}&=2C_3^{(k-1)} \\
C_3^{(k)}&= \frac{8}{\eta^2}C_3^{(k-1)}&C_4^{(k)}&= \frac{2}{\eta}C_3^{(k-1)}\enspace.
\end{align*}
We can now see that 
\begin{align*}
&C_2^{(k)}= 2\left(\frac{8}{\eta^2}\right)^{k-1} = 2(8(1+\alpha)^{2k})^{k-1}\\
&C_3^{(k)}=\left(\frac{8}{\eta^2}\right)^k=(8(1+\alpha)^{2k})^{k}  \\
&C_4^{(k)}= \frac{2}{\eta}\left(\frac{8}{\eta^2}\right)^{k-1}= 2(1+\alpha)^k(8(1+\alpha)^{2k})^{k-1}
\enspace.
\end{align*}
Finally, we can get a lower bound on the value of $C_1^{(k)}$ as follows:
\begin{align*}
C_1^{(k)}
&=\prod_{i=2}^k 2^{-2iC_3^{(i-1)}}
%&=2^{-2kC_3^{(k-1)}}2^{-2(k-1)C_3^{(k-2)}}\dots2^{-2\cdot 2 C_3^{(1)}}\\
%&=2^{-2\sum_{i=2}^k i\cdot  C_3^{(i-1)}}\\
=2^{-2\sum_{i=2}^k i\cdot  \left(\frac{8}{\eta^2}\right)^{(i-1)}}
\geq 2^{-2k^2\left(\frac{8}{\eta^2}\right)^{(k-1)}} 
= 2^{-2k^2(8(1+\alpha)^{2k})^{k-1}}
\enspace.
\end{align*}
\end{proof}

\section{An Adversarial Sample Complexity Lower Bound}

In this section, we will show that any robust learning algorithm for
monotone conjunctions under the uniform distribution must have an
exponential sample-complexity dependence on an adversary's budget
$\rho$.  This result extends to any superclass of monotone
conjunctions, such as CNF formulas, decision lists and halfspaces.  It
is a generalization of Theorem~13 in \cite{gourdeau2021hardness},
which shows that no sample-efficient robust learning algorithm exists
for monotone conjunctions against adversaries that can perturb
$\omega(\log(n))$ bits of the input under the uniform distribution.

The idea behind the proof is to show that, for a fixed constant $\kappa<2$, and sufficiently large input dimension, a sample of size $2^{\kappa\rho}$ from the uniform distribution won't be able to distinguish between two disjoint conjunctions of length $2\rho$. 
However, the robust risk between these two conjunctions can be lower bounded by a constant.
Hence, there does not exist a robust learning algorithm with sample complexity 
$2^{\kappa\rho}$ that works for the uniform distribution, and arbitrary input dimension and confidence and accuracy parameters.

Recall that the sample complexity of PAC learning conjunctions is
$\Theta(n)$ in the non-adversarial setting.  On the other hand, our
adversarial lower bound in terms of the robust parameter is super
linear in $n$ as soon as the adversary can perturb more than
$\log(\sqrt{n}))$ bits of the input.

\begin{theorem}
\label{thm:mon-conj}
Fix a positive increasing robustness function $\rho:\N\rightarrow\N$. 
For $\kappa<2$ and sufficiently large input dimensions $n$, any $\rho(n)$-robust learning algorithm for {\MonConj} has a sample complexity lower bound of $2^{\kappa\rho(n)}$ under the uniform distribution.
\end{theorem}

The proof of the theorem follows similar reasoning as Theorem~13 in
\cite{gourdeau2021hardness}, and is included in
Appendix~\ref{app:sc-lb}.  The main difference in the proof is its
reliance on the following lemma, which shows that, for sufficiently
large input dimensions, a sample of size $2^{\kappa\rho}$ from the
uniform distribution will look constant with probability $1/2$ if
labelled by two disjoint monotone conjunctions of length $2\rho$.  As
shown in Lemma~\ref{lemma:bound-loss}, which can be found in
Appendix~\ref{app:sc-lb}, these two conjunctions have a robust risk
bounded below by a constant against each other.

\begin{lemma}
\label{lemma:concepts-agree}
For any constant $\kappa<2$, for any robustness parameter $\rho\leq n/4$, for any disjoint monotone conjunctions $c_1,c_2$ of length $2\rho$, there exists $n_0$ such that for all $n\geq n_0$, a sample $S$ of size $2^{\kappa\rho}$ sampled i.i.d. from $D$ will have that $c_1(x)=c_2(x)=0$ for all $x\in S$ with probability at least $1/2$.
\end{lemma}

\begin{proof}
We begin by bounding the probability that $c_1$ and $c_2$ agree on an
i.i.d. sample of size $m$.  We have
\begin{equation}
\label{eqn:zero-label-sample}
\Prob{S\sim D^m}{\forall x\in S  \cdot c_1(x)=c_2(x)=0}
%&=\left(\Prob{x\sim D}{x\in X_{00}}\right)^m\\
=\left(1-\frac{1}{2^{2\rho}}\right)^{2m}
\enspace.
\end{equation}
In particular, if 
\begin{equation}
\label{eqn:sample-size}
m\leq \frac{\log(2)}{2\log(2^{2\rho}/(2^{2\rho}-1))}
\enspace,
\end{equation}
then the RHS of  Equation~\ref{eqn:zero-label-sample} is at least $1/2$.

Now, let us consider the following limit, where $\rho$ is a function of the input parameter $n$:
\begin{align*}
\underset{n\rightarrow\infty}{\lim}\;
2^{\kappa\rho}\log\left(\frac{2^{2\rho}}{2^{2\rho}-1}\right)
&=\frac{-\log(4)}{\kappa\log(2)}\;\underset{n\rightarrow\infty}{\lim}\;\frac{2^{\kappa\rho}}{1-2^{2\rho}}   \\
&=\frac{-\log(4)}{\kappa\log(2)}\;\frac{\kappa\log(2)}{-2\log(2)}\underset{n\rightarrow\infty}{\lim}\;\frac{2^{\kappa\rho}}{2^{2\rho}}   \\
&= \underset{n\rightarrow\infty}{\lim}\;2^{(\kappa-2)\rho}\\
&=
\begin{cases} 
0 & \text{if $\kappa<2$} \\
1 & \text{if $\kappa=2$} \\
\infty & \text{if $\kappa>2$}
\end{cases}
\enspace,
\end{align*}
where the first two equalities follow from l'H\^opital's rule.

Thus if $\kappa<2$ then $2^{\kappa\rho}$ is $o\left(\left(\log\left(\frac{2^{2\rho}}{2^{2\rho}-1}\right)\right)^{-1}\right)$.

\end{proof}

\begin{remark}
  Note that for a given $\kappa<2$, the lower bound $2^{\kappa\rho}$
  holds only for sufficiently large $\rho(n)$.  By looking at
  Equation~\ref{eqn:zero-label-sample}, and letting $m=2^\rho$,
  we get that $\rho(n)\geq2$ is a sufficient condition for it to hold.  
  If we want a lower bound for robust learning that
  is larger than that of standard learning (where the dependence is
  $\Theta(n)$) for a $\log(n)$ adversary, setting $m=2^{1.7\rho}$
  and requiring $\rho(n)\geq6$, for e.g., would be sufficient.
\end{remark}

\section{Conclusion}

We have shown that the class $k$-DL is efficiently robustly learnable
against a logarithmically-bounded adversary, thus making progress on
the open problem of \citet{gourdeau2021hardness} of whether
PAC-learnable classes are always robust in general against a
logarithmically-bounded adversary.  The main technical tool was an
isoperimetric result concerning CNF formulas.  Moreover, we have shown
that, for monotone conjunctions and any superclass thereof, any
$\rho$-robust learning algorithm must have a sample complexity that is
exponential in the adversarial budget $\rho$.

Deriving sample complexity bounds for the robust learnability of
halfspaces under the uniform distribution is perhaps the most natural
next step towards resolving the above-mentioned open problem.  Another
direction of further research concerns improving the sample complexity
bounds for $k$-DL in the present paper.  Here we have used a proper
PAC-learning algorithm as a black box in our robust learning procedure (see Corollary~\ref{cor:k-dl}).
By controlling the accuracy parameter of the standard PAC-learning
algorithm, we are able to get a robust learning algorithm.  From this,
we get polynomial sample complexity upper bounds for $k$-DL in terms
of the \emph{robustness} accuracy parameter $\varepsilon$, the
distribution parameter $\alpha$, and the input dimension $n$.  The
resulting polynomial has degree $O(8^k(1+\alpha)^{2k^2})$ in the term
$1/\varepsilon$ and degree $O(k8^k(1+\alpha)^{2k^2})$ in the dimension
$n$.  It is natural to ask whether these bounds can be improved in a
significant way, e.g., by adapting the learning procedure to directly
take robustness into account, rather than using a PAC-learning
algorithm as a black box.  Connected to this, we note that our lower bound
focuses on establishing the exponential dependence of the number of
samples on the robustness parameter.  The bound is derived from the
case of monotone conjunctions (a special case of 1-DL) under the
uniform distribution and so does not mention $k$, nor the
distribution parameter $\alpha$.  Likewise, it does not mention the
desired accuracy $\varepsilon$.  Deriving sample complexity lower
bounds with a dependence on these parameters, potentially through
other techniques, would help give a complete picture of the robust
learnability of $k$-DL.

\section*{Acknowledgments}
MK and PG received funding from the ERC under the European Union’s Horizon 2020 research and innovation programme (FUN2MODEL, grant agreement No.~834115).

\bibliographystyle{apalike}
\bibliography{references}

%\newpage
\appendix

%\section{Appendix}
\section{Preliminaries}
\subsection{The PAC framework}
\label{app:pac}

We study the problem of robust classification in  the realizable setting and where the input space is the Boolean cube $\X_n=\{0,1\}^n$. 
For clarity, we first recall the definition of the PAC learning framework \cite{valiant1984theory}.

\begin{definition}[PAC Learning]
Let $\C_n$ be a concept class over $\X_n$ and let $\C=\bigcup_{n\in\N}\C_n$.
We say that $\C$ is \emph{PAC learnable using hypothesis class $\H$} and sample complexity function $p(\cdot,\cdot,\cdot,\cdot)$ if there exists an algorithm $\A$ that satisfies the following:
for all $n\in\N$, for every $c\in\C_n$, for every $D$ over $\X_n$, for every $0<\epsilon<1/2$ and $0<\delta<1/2$, if whenever $\A$ is given access to $m\geq p(n,1/\epsilon,1/\delta,\text{size}(c))$ examples drawn i.i.d. from $D$ and labeled with $c$, $\A$ outputs a polynomially evaluatable $h\in\H$ such that with probability at least $1-\delta$, 
\begin{equation*}
\Prob{x\sim D}{c(x)\neq h(x)}\leq \epsilon\enspace.
\end{equation*}
We say that $\C$ is statistically efficiently PAC learnable if $p$ is polynomial in $n,1/\epsilon$, $1/\delta$ and size$(c)$, and computationally efficiently PAC learnable if $\A$ runs in polynomial time in $n,1/\epsilon$, $1/\delta$ and size$(c)$.
\end{definition}

PAC learning is \emph{distribution-free}, in the sense that no assumptions are made about the distribution from which the data comes from.
The setting where $\C=\H$ is called \emph{proper learning}, and \emph{improper learning} otherwise.

\subsection{Monotone Conjunctions and $k$-CNF Formulas}
\label{app:formulas}

%\begin{definition}[Conjunctions]
A conjunction $c$ over $\{0,1\}^n$ can be represented a set of literals $l_1,\dots,l_k$, where, for $x\in\X_n$, $c(x)=\bigwedge_{i=1}^k l_i$. 
For example, $c(x)=x_1\wedge\bar{x_2}\wedge{x_5}$ is a conjunction.
Monotone conjunctions are the subclass of conjunctions where negations are not allowed, i.e., all literals are of the form $l_i=x_j$ for some $j\in[n]$.

A formula $\varphi$ in the conjunctive normal form (CNF) is a conjunction of clauses, where each clause is itself a disjunction of literals. 
A $k$-CNF formula is a CNF formula where each clause contains at most $k$ literals.
For example, $\varphi= (x_1 \vee x_2) \wedge (\bar{x_3} \vee x_4) \wedge \bar{x_5}$ is a 2-CNF.

\subsection{Log-Lipschitz Distributions}
\label{app:log-lipschitz}

Log-Lipschitz distributions have the following useful properties, which are stated in \cite{awasthi2013learning} and whose proofs can be found in \cite{gourdeau2019hardness}:
\begin{lemma}
\label{lemma:log-lips-facts}
Let $D$ be an $\alpha$-$\log$-Lipschitz distribution over $\boolhc$. 
Then the following hold:
\begin{enumerate}[i.]
\item\label{test} For $b\in\{0,1\}$, $\frac{1}{1+\alpha}\leq \Prob{x\sim D}{x_i=b}\leq\frac{\alpha}{1+\alpha}$.
\item For any $S\subseteq[n]$, the marginal distribution $D_{\bar{S}}$ is $\alpha$-$\log$-Lipschitz, where $D_{\bar{S}}(y)=\sum_{y'\in\{0,1\}^S} D(yy')$.
\item For any $S\subseteq[n]$ and for any property $\pi_S$ that only depends on variables $x_S$, the marginal with respect to $\bar{S}$ of the conditional distribution $(D|\pi_S)_{\bar{S}}$ is $\alpha$-$\log$-Lipschitz.
\item For any $S\subseteq[n]$ and $b_S\in\{0,1\}^S$,  we have that $\left(\frac{1}{1+\alpha}\right)^{|S|}\leq \Prob{x\sim D}{x_i=b}\leq\left(\frac{\alpha}{1+\alpha}\right)^{|S|}$.
\end{enumerate}
\end{lemma}

\section{Proof of Corollary~\ref{cor:k-dl}}
\label{app:log-exp}

\begin{proof}[Proof of Corollary~\ref{cor:k-dl}]
  Let $\A$ be the (proper) PAC-learning algorithm for k-DL as
  in~\cite{rivest1987learning}, with sample complexity $\poly(\cdot)$.
  Fix the input dimension $n$, target concept $c$ and distribution
  $D\in \D_n$, and let $\rho=\log n$.  Fix the accuracy parameter
  $0<\varepsilon<1/2$ and confidence parameter $0<\delta<1/2$ and let
  $\eta=1/(1+\alpha)^k$. Set 
  $$\varepsilon_0=C_1\left(\frac{16\varepsilon}{e^4n^{2k+2}}\right)^{C_2}
  \min\set{\left(\frac{16\varepsilon}{e^4n^{2k+2}}\right)^{C_3},n^{-C_4}}\enspace,$$ 
  where the constants are the ones derived in Theorem~\ref{thm:k-cnf}. 
  
  Let $m=\lceil\poly(n,1/\delta,1/\varepsilon_0)\rceil$, and note that
  $m$ is polynomial in $n$, $1/\delta$ and $1/\varepsilon$.

  Let $S\sim D^m$ and $h=\A(S)$.  Let the target and hypothesis be 
  defined as the following decision lists: $c=((K_1,v_1),\ldots,(K_r,v_r))$ and
	$h=((K'_1,v'_1),\ldots,(K'_s,v'_s))$, where the clauses $K_i$ are 
	conjunctions of $k$ literals. Given $i\in\{1,\ldots,r\}$ and 
	$j \in \{1,\ldots,s\}$, define a $k$-CNF formula $\varphi^{(c,h)}_{i,j}$ 
	by writing
\[ \varphi^{(c,h)}_{i,j} = \neg K_1 \wedge \cdots \wedge \neg
  K_{i-1} \wedge K_i \wedge  \neg K'_1 \wedge \cdots \wedge \neg
  K'_{j-1}\wedge K'_j \, . \]
   Notice that the formula $\varphi^{(c,h)}_{i,j}$ represents the set of
   inputs $x\in \X$ that respectively activate vertex $i$ in $c$ and
   vertex $j$ in $h$.
   
Since $ \Prob{x\sim D}{h(x)\neq c(x)}<\varepsilon_0$ with probability at least
$1-\delta$, any $\varphi^{(c,h)}_{i,j}$ that leads to a misclassification must 
have $\satnot(\varphi^{(c,h)}_{i,j})<\varepsilon_0$.
But by Theorem~\ref{thm:k-cnf}, $\satlog(\varphi^{(c,h)}_{i,j})<\frac{16\varepsilon}{e^4n^{2k+2}}$ for all $\varphi^{(c,h)}_{i,j}$ with probability at least
$1-\delta$. 

Hence the probability that a $\rho$-bounded adversary can
make $\varphi^{(c,d)}_{i,j}$ true  is at most $\frac{16\varepsilon}{e^4n^{2k+2}}$.  
Taking a union bound over all possible choices of $i$ and $j$ (there are 
$\sum_{i=1}^k{n\choose k}\leq k\left(\frac{en}{k}\right)^k$ possible clauses
in $k$-decision lists, which gives us a crude estimate of 
$k^2\left(\frac{en}{k}\right)^{2k}\leq \frac{e^4n^{2k+2}}{16} $ choices of 
$i$ and $j$) we conclude that
$\roblog(h,c) < \varepsilon$.

\end{proof}

\section{Proof of Theorem~\ref{thm:mon-conj}}
\label{app:sc-lb}

The proof of Theorem~\ref{thm:mon-conj} relies on the following lemmas:
\begin{lemma}[Lemma 6 in \cite{gourdeau2021hardness}]
\label{lemma:robloss-triangle}
Let $c_1,c_2\in\{0,1\}^\X$ and fix a distribution on $\X$. 
Then for all $h:\boolhc\rightarrow\set{0,1}$
\begin{equation*}
\robloss(c_1,c_2)
\leq \robloss(h,c_1)+\robloss(h,c_2)
\enspace.
\end{equation*}
\end{lemma}

We then recall the following lemma from \cite{gourdeau2021hardness}, whose proof here makes the dependence on the adversarial budget $\rho$ explicit.

\begin{lemma}
\label{lemma:bound-loss}
Under the uniform distribution, for any $n\in\N$, disjoint $c_1,c_2\in{\MonConj}$ of even length $3\leq l\leq n/2$  on $\boolhc$ and robustness parameter $\rho= l/2$, we have that $\robloss(c_1,c_2)$ is bounded below by a constant that can be made arbitrarily close to $\frac{1}{2}$ as $l$ (and thus $\rho$) increases. 
\end{lemma}
\begin{proof}
For a  hypothesis $c\in{\MonConj}$, let $I_c$ be the set of variables in $c$.
Let $c_1,c_2\in\C$ be as in the theorem statement.
Then the robust risk $\robloss(c_1,c_2)$ is bounded below by 
\begin{equation*}
\Prob{x\sim D}{c_1(x)=0\wedge x\text{ has at least $ \rho $ 1's in }I_{c_2}}\geq (1-2^{-2\rho})/2\enspace.
\end{equation*}
\end{proof}

\begin{proof}[Proof of Theorem~\ref{thm:mon-conj}]
Fix any algorithm $\A$ for learning {\MonConj}.
We will show that the expected robust risk between a randomly chosen target function and any hypothesis returned by $\A$ is bounded below by a constant.
Let $\delta=1/2$, and fix a positive increasing adversarial-budget function $\rho(n)\leq n/4$ ($n$ is not yet fixed).
Let $m(n)=2^{\kappa\rho(n)}$ for an arbitrary $\kappa<0$.
Let $n_0$ be as in Lemma~\ref{lemma:concepts-agree}, where $m(n)$ is the fixed sample complexity function. 
Then Equation~(\ref{eqn:sample-size}) in the proof of Lemma~\ref{lemma:concepts-agree} holds for all $n\geq n_0$.

Now, let $D$ be the uniform distribution on $\boolhc$ for $n\geq \max(n_0,3)$, and choose $c_1$, $c_2$ as in Lemma~\ref{lemma:bound-loss}.
Note that $\robloss(c_1,c_2)>\frac{5}{12}$ by the choice of $n$.
Pick the target function $c$ uniformly at random between $c_1$ and $c_2$, and label $S\sim D^{m(n)}$ with $c$.
By Lemma~\ref{lemma:concepts-agree}, $c_1$ and $c_2$ agree with the labeling of $S$ (which implies that all the points have label $0$) with probability at least~$\frac{1}{2}$ over the choice of $S$. 

Define the following three events for $S\sim D^m$: 
\begin{align*}
&\E:\;{c_1}_{|S}={c_2}_{|S}\;,\enspace
\E_{c_1}:\;c=c_1\;,\enspace
\E_{c_2}:\;c=c_2\enspace.
\end{align*}

Then, by Lemmas~\ref{lemma:concepts-agree} and~\ref{lemma:robloss-triangle}, 
\begin{align*}
\eval{c,S}{\robloss(\A(S),c)}
&\geq\Prob{c,S}{\E}\eval{c,S}{\robloss(\A(S),c)\;|\;\E}\\
&>\frac{1}{2}(
\Prob{c,S}{\E_{c_1}}\eval{S}{\robloss(\A(S),c)\;|\;\E\cap\E_{c_1}}+\Prob{c,S}{\E_{c_2}}\eval {S}{\robloss(\A(S),c)\;|\;\E\cap \E_{c_2}}
)\\
&=\frac{1}{4}\;\eval{S}{\robloss(\A(S),c_1)+\robloss(\A(S),c_2)\;|\;\E}\\
&\geq \frac{1}{4}\;\eval{S}{\robloss(c_2,c_1)}\\
&=\frac{5}{48}
\enspace.
\end{align*}

\end{proof}

\end{document}